\documentclass[manuscript,screen]{acmart}
\AtBeginDocument{%
  \providecommand\BibTeX{{%
    \normalfont B\kern-0.5em{\scshape i\kern-0.25em b}\kern-0.8em\TeX}}}

\setcopyright{acmcopyright}
\copyrightyear{2022}
\acmYear{2022}

\renewcommand\footnotetextcopyrightpermission[1]{}

\usepackage{amsmath}
\usepackage{algorithm}
\usepackage{algpseudocode}
\usepackage{bbm}
\algdef{SE}[SUBALG]{Indent}{EndIndent}{}{\algorithmicend\ }%
\algtext*{Indent}
\algtext*{EndIndent}

\newtheorem{theorem}{Theorem}

\DeclareMathOperator*{\argmax}{arg\,max}

\begin{document}

\title{Adaptive Sampling Strategies to Construct Equitable Training Datasets}

\author{William Cai}
\email{willcai@stanford.edu}
\affiliation{%
  \institution{Stanford University}
  \country{USA}
}

\author{Ro Encarnacion}
\affiliation{%
  \institution{Stanford University}
  \country{USA}
}

\author{Bobbie Chern}
\affiliation{%
  \institution{Meta}
  \country{USA}
}

\author{Sam Corbett-Davies}
\affiliation{%
  \institution{Meta}
  \country{USA}
}

\author{Miranda Bogen}
\affiliation{%
  \institution{Meta}
  \country{USA}
}

\author{Stevie Bergman}
\affiliation{%
  \institution{Meta}
  \country{USA}
}

\author{Sharad Goel}
\email{sgoel@hks.harvard.edu}
\affiliation{%
  \institution{Harvard University}
  \country{USA}
}
\renewcommand{\shortauthors}{Cai et al.}

\begin{abstract}
In domains ranging from computer vision to natural language processing, machine learning models have been shown to exhibit stark disparities, often performing worse for members of traditionally underserved groups. 
One factor contributing to these performance gaps is a lack of representation in the data the models are trained on.  
It is often unclear, however, how to operationalize representativeness in specific applications. 
Here we formalize the problem of creating equitable training datasets, and 
propose a statistical framework for addressing this problem.
We consider a setting where a model builder must decide how to allocate a fixed data collection budget to gather training data from different subgroups.
We then frame dataset creation as a constrained optimization problem, in which one maximizes a function of group-specific performance metrics based on (estimated) group-specific learning rates and costs per sample.
This flexible approach incorporates preferences 
of model-builders and other stakeholders,
as well as the statistical properties of the learning task.
When data collection decisions are made sequentially, 
we show that under certain conditions this optimization problem can be efficiently solved even without prior knowledge of the learning rates. 
To illustrate our approach, we 
conduct a simulation study of polygenic risk scores on synthetic genomic data---an application domain that often suffers from non-representative data collection. 
We find that our adaptive sampling strategy outperforms several common data collection heuristics, including equal and proportional sampling, demonstrating the value of strategic dataset design for building equitable models. 
\end{abstract}

\maketitle

\section{Introduction}

Consider the problem of building a computer vision model to detect deforestation from satellite imagery~\citep{maretto2020spatio,irvin2020forestnet,hardt2016equality}. 
Such models may be useful to assess ecological damage, and to guide the investment of resources by 
government agencies, legal organizations, and environmental groups.
Machine learning models like this---as well as related models in natural language processing, healthcare, criminal justice, and beyond---have been shown to exhibit sharp disparities, often performing worse on subgroups of the population defined by race, ethnicity, gender, language, and nationality~\citep{buolamwini2018gender, koenecke2020racial,blodgett2016demographic,sap2019risk,caliskan2017semantics,de2019bias,chouldechova2017fair,kleinberg_inherent_2017,corbett2017algorithmic,goodman2018machine, obermeyer2019dissecting}.
Our deforestation model might, for instance, perform worse in certain regions of the world, perhaps given differences in the visual appearance of the tree canopy.
A variety of techniques in the fair machine learning community attempt to mitigate such shortcomings~\citep{corbett2018measure, zafar2017parity, dwork2012fairness, coston2020counterfactual, mishler2021fairness, ristanoski2013discrimination, berk2017convex, zafar2017fairness, calders2010three, fish2016confidence, kamiran2013quantifying, verma2018fairness, kleinberg_inherent_2017, kusner2017counterfactual, kamishima2012fairness, corbett2017algorithmic}.  
For example, one might constrain the computer vision model to have equal error rates across countries~\citep{maretto2020spatio,irvin2020forestnet,hardt2016equality}.
Popular approaches to algorithmic fairness---such as demanding error rate parity---often implicitly assume a fixed training dataset, 
with disparities addressed by altering the statistical model.
In many cases, however, it is also possible to update the training datasets themselves (e.g., one might seek out or label additional satellite images from certain countries), 
and so it is important to design approaches to algorithmic fairness that consider this possibility.

There have been numerous calls to make datasets more diverse~\citep{ai2019high, raji2019actionable,gebru2021datasheets, mitchell2019model,act2021proposal}, including by policymakers, but, in practice, it is often 
unclear how exactly one should compile datasets in specific domains to ensure the models that are trained on them are broadly equitable.
For instance, following the heuristic of ``equal sampling,''  one might label an equal number of images per country;
alternatively, following the heuristic of ``representative sampling,'' one might label images in proportion to the geographic area of each country.
While both strategies aim for diversity in the dataset, 
they can lead to quite different outcomes and downstream models.
Furthermore, neither of these sampling strategies directly considers either the costs of data collection or the impact of datapoints on model performance.
For example, if the costs of data collection vary across countries, then, given a fixed budget, different sampling strategies can lead to different total dataset sizes, impacting overall model performance.
There may similarly be variance in how much datapoints from one country impact model performance in other regions (e.g., due to similarity in vegetation). 
Relatedly, it may be important to prioritize performance in certain regions (e.g., to maximize impact given local intended uses and expected impacts of the model, regulatory constraints, or because those areas have been historically neglected), creating additional considerations for dataset construction.

Here we develop a framework for constructing broadly equitable datasets and for evaluating the equity of existing datasets. 
We start, in Section \ref{setup}, by formalizing the dataset construction problem in a way that accounts for both the costs and consequences of data collection strategies. 
Our approach separates the task into two key components. First, we introduce the notion of group-specific ``learning curves'' that describe how the allocation of training data affects the resulting group-level model performance.
For instance, in our computer vision application, the model performance might be high in one country even with relatively small amounts of labeled data in that country, whereas in another country more samples from that country might be required to achieve a comparably high model performance. 
Second, we incorporate the model-builder's preferences over the resulting group-level model performances into a utility function.
For example, a model-builder might specify 
how to prioritize performance across regions.
Given these two ingredients, dataset construction can then be framed as maximizing utility subject to the budgetary constraints.  

When the learning curves are fully known and concave, and the utility is linear, this formalization results in a convex optimization problem which can be efficiently solved via standard approaches.  
But in most cases, the model-builder does not have a priori knowledge of the learning curves,
creating additional challenges for efficiently constructing datasets that are appropriate to the task.
In Section \ref{sec:opt}, 
we consider a setting in which datasets are constructed sequentially---a setting that is common when datapoints are labeled online, with a remote workforce.
In this case, we present an adaptive sampling algorithm that can, under certain conditions, efficiently find a utility-maximizing allocation even in the absence of knowledge of the learning curves.
Both analytically and empirically, we show that our adaptive algorithm gives near-optimal performance in a variety of scenarios,
outperforming common alternatives.

Finally, in Section \ref{sec:prs} we evaluate our approach using a popular dataset simulator 
used
to train polygenic risk score (PRS) models, 
which seek to identify high-risk individuals for targeted health interventions via genomic data. 
We construct a hypothetical disease and health intervention, 
and evaluate how changing the allocation of training data between more or fewer people of European and African descent affects the quality of risk stratification for the intervention in both populations. 
We find that our adaptive approach to sampling allows for model-builders to construct models which maximize the total impact of the recommended health intervention while allowing them the flexibility to efficiently increase impact in groups traditionally excluded by PRSs.  

\vspace{-2mm}
\section{Related Work}
Given a fixed dataset, the problem of training fair models has received considerable attention from the machine learning community~\citep{corbett2018measure, zafar2017parity, dwork2012fairness, coston2020counterfactual, mishler2021fairness, ristanoski2013discrimination, berk2017convex, zafar2017fairness, calders2010three, fish2016confidence, kamiran2013quantifying, verma2018fairness, kleinberg_inherent_2017, kusner2017counterfactual, kamishima2012fairness, corbett2017algorithmic},
where many of the popular approaches fall into one of two broad categories: equalizing error rates between groups or minimizing the impact of sensitive attributes on downstream predictions.
There is a substantially smaller literature on 
the equitable construction of datasets,
in which a model-builder can choose how to allocate a fixed budget to acquire training samples from different groups to mitigate inequities. 
Below we briefly describe some of the most related research in this line of work.

\citet{branchaud2021can} consider how BALD---a heuristic algorithm for active learning which searches for the most informative datapoint overall to sample next without knowledge of group membership---improved accuracy of minority group model performance and predictive parity compared to uniform sampling.  
However, this heuristic does not take into account tradeoffs in group-level model performances when sampling, so it is ill-suited to solve our allocation problem; 
we note, though, that adding an active learning subroutine to our sampling approach could be a promising direction for future work. 

\citet{anahideh2020fair} and \citet{sharaf2020promoting} both propose group-aware active learning techniques to the problem of allocating a fixed budget to sample a dataset from different groups.  
In both methodologies, the model-builder identifies a fairness metric and uses an active learning framework to select samples which both improve the overall model performance along with the fairness metric.  
Our approach differs from these fundamentally in that we specify the model-builder's utility directly in terms of group-level model performances, centering the consequences of performance disparities~\citep{L2BF,nilforoshan2022}.
We show that our specification has numerous upsides, making it straightforward to:
(1) implement interventions to make models more inclusive to traditionally underserved groups, beyond satisfying a fairness metric;
(2) adaptively sample under a wider range of learning curves; and
(3) audit a dataset for inclusivity.  
Finally, \citet{abernethy2021active} propose a max-min fairness theory for active sampling by sampling from the group that currently has the worst model performance at each step. 
We evaluate this strategy in our work and characterize how it compares to other strategies for constructing equitable datasets;
in particular, we show that it can lead to sub-optimal results, as it does not consider the \emph{rate} at which datapoints improve performance.

In addition to the algorithmic approaches described above, 
there have been many real-world efforts to compile more inclusive datasets in several different domains~\citep{matise2011next, aviddataset,galvez2021people,hazirbas2021towards}.
These efforts often employ a variety of natural heuristics---for example, ensuring a minimum level of representation across groups.
Such heuristics are often useful when downstream applications are varied or less well specified, but, as we show, they can be sub-optimal for specific, well-defined modeling tasks. 
Finally, given an existing dataset,
many proposals have suggested ways to characterize their equity and aid appropriate use, for example by including statements describing what populations the datasets are and are not representative of~\citep{gebru2021datasheets, bender-friedman-2018-data, holland2018dataset, hind2018increasing}.

\section{Problem setup}\label{setup}
\subsection{A model of sampling}
We consider a scenario where a model builder has a fixed budget $B$ which they can use to obtain training data associated with $K$ different groups.\footnote{The ``budget'' can include both monetary and other costs, such as time or effort.}
Let $c_k$ denote the cost for obtaining a single sample from group $k$.  
Returning to our running example, a researcher training a computer vision model to detect deforestation must decide how to allocate their budget to obtain labeled satellite images from $K$ different countries, with country-specific costs of data collection $c_k$.

To formalize the model-builder's allocation problem, we next introduce the idea of group-specific ``learning curves'', which capture the expected performance gains under different sampling strategies.
Let $\vec{n} = (n_1, ..., n_K)$ describe the number of samples collected from each group under a given strategy, where $n_k \in \mathbb{R}_+$.
Note that we allow fractional sample sizes---not just integer sizes---which we interpret as a probabilistic strategy.
Specifically, 
if $n_k = u + v$ for an integer $u$ and $0 < v < 1$, 
then after collecting the first $u$ datapoints, an additional datapoint is collected with 
probability $v$.
To satisfy the budget constraint (in expectation), we require that $\sum_{k=1}^K c_k n_k \leq B$.  

Now let 
\begin{equation}
T_{\vec{n}} =\left \{(X^1, Y^1), (X^2, Y^2), ... (X^N, Y^N) \right \}
\end{equation}
denote
a random training dataset with features $X$ and labels $Y$ satisfying the given allocation $\vec{n}$.
In particular, if
$X_g^i$ denotes the group membership of the $i$-th datapoint,
then, for $1 \leq k \leq K$, $\mathbb{E}[|\{i : X_g^i = k\}|] = n_k$.

We further assume that within each group, 
the samples $(X,Y)$ are i.i.d draws from a fixed, group-specific data-generating distribution.

Suppose that 
$\hat{f}_{T_{\vec{n}}}$ is a model fit to the training data $T_{\vec{n}}$,
with $\hat{f}_{T_{\vec{n}}}(X^0)$
denoting the model prediction on a datapoint $X^0$.
Then, the group-level model performance given the training dataset $T_{\vec{n}}$ is
\begin{equation}
    \text{PERF}_{k, T_{\vec{n}}} = \mathbb{E}_{X^0, Y^0}[G(Y^0, \hat{f}_{T_{\vec{n}}}(X^0) \mid X^0_g=k],
\end{equation}
where $G$ is defined by the model-builder to be a measure of model performance given prediction $\hat{f}_{T_{\vec{n}}}(X^0)$ and true outcome $Y^0$.
The group-level performance is thus the expected model performance, as defined by $G$, of the model for a new point $(X^0, Y^0)$ drawn from the joint distribution of the data conditioned on membership in group $k$.

For example, in our setting we might define $G = a Y^0\hat{f}_{T_{\vec{n}}}(X^0) - b(1-Y^0)\hat{f}_{T_{\vec{n}}}(X^0)$ for some positive constants $a$ and $b$ which balance the benefit of detecting a true instance of deforestation versus the cost of a false positive, respectively.

Finally, the expected group-level model performance given a training allocation $\vec{n}$ is
\begin{equation}
    M_k(\vec{n}) = \mathbb{E}_{T_{\vec{n}}}\mathbb{E}_{X^0, Y^0}[G(Y^0, \hat{f}_{T_{\vec{n}}}(X^0)) \mid X^0_g=k],
\end{equation}
where the outer expectation is taken over random datasets satisfying the specified group-level sample sizes.
We call this function $M_k$ the group-level learning curve, the function which maps a training allocation to the expected model performance, 
and let $\vec{M}(\vec{n}) = (M_1(\vec{n}), ..., M_K(\vec{n}))$ denote the vector of group-level performances for each group.

For a given learning curve, we now define 
a model-builder's utility over different allocations.
This utility can be written as $U(\vec{M}) = U(M_1(\vec{n}), ..., M_K(\vec{n}))$,
where $U(\vec{M})$ can be thought of as the model-builder's preference over model performances for different groups.  
In some settings, the model-builder may wish to prioritize model performance in one particular group: for instance, in our deforestation example, a researcher may wish to prioritize performance in a country with a stronger  regulatory environment which can better translate model performance to impact, or a country which has been traditionally understudied by other deforestation analyses.  
To capture such preferences, we primarily consider utility functions that are a linear combination of the model performances of each group, of the form 
\begin{equation}\label{eq:utility}
U(\vec{M}) = \sum_{k=1}^K a_kM_k,
\end{equation}
where $a_k \geq 0$.  
In particular, this specification allows the model builder the flexibility to intervene to make models more inclusive, by setting $a_k$ higher for groups which, for example, have been traditionally excluded.  We note that these groups need not have lower model performance in order to be prioritized, distinguishing our approach from those aiming for performance parity. 

Finally, given the above ingredients, 
the model-builder's optimization problem is to choose an allocation $\vec{n}^*$ which maximizes utility subject to the budget constraint:
\begin{align}
\label{eq:opt}
\begin{aligned}
& \vec{n}^* \in \argmax_{\vec{n}} U(M_1(\vec{n}), ..., M_K(\vec{n})) \\
& \text{subject to:} \ \sum_{k=1}^K c_kn_k \leq B.
\end{aligned}
\end{align}
If the learning curves $M_k$ are known and concave, and the utility function is linear---as in Eq.~\eqref{eq:utility}---then $U(\vec{M}(\vec{n}))$ is itself a concave function of $\vec{n}$.
More generally, if $U$ is concave and non-decreasing in every element of $\vec{M}$, then, since a concave non-decreasing function of concave functions is itself concave, $U(\vec{M}(\vec{n}))$ is concave.
In these cases, an optimal allocation $\vec{n}^*$
can be efficiently computed using off-the-shelf tools for convex optimization.
In Section~\ref{sec:opt}, we develop an alternative approach to finding optimal allocations that does not require full knowledge of the learning curves.
 
In addition to formalizing the problem of dataset construction, this framework provides an approach for auditing existing datasets. 
Specifically, for an auditor who might have their own preference $\tilde{U}(M_1(\vec{n}), ..., M_k(\vec{n}))$, they can estimate the gap
\begin{equation}
\max_{\vec{n}} \tilde{U}(\vec{n}) - \tilde{U}\left(\vec{n}^*_\text{model-builder}\right),
\end{equation}
where the maximum is taken over feasible allocations.
A large gap suggests that the model-builder's implied preferences over group-level model performances, based on their allocation, differs from that of the auditor's.  

\subsection{Alternative specifications of the utility function}

In the linear specification of utility introduced above, the model builder's preferences for a single group do not depend on how well the model performs for other groups.
Alternative specifications might allow for a direct penalty to inequality: for example, we could specify 
\begin{equation}
\label{eq:parity-utility}
U(\vec{M}) = \sum_{k=1}^K a_kM_k - b|M_k - \overline{M}|,
\end{equation}
where $\overline{M}$ denotes the average performance across all groups, and the penalization term $|M_k - \overline{M}|$ signals that the model builder prefers a solution where the model performances across groups are similar.  
For example, if an allocation results in the same model performance for all groups, such that $M_k = \overline{M}$ for all $k$, then the penalization term is 0.  

In some cases, explicitly encoding preferences for parity can be appropriate to the application. 
In other instances, though, doing so can lead to unintended consequences.
For example, 
consider two possible allocations $\vec{n}_1$ and $\vec{n}_2$ over three groups such that
$\vec{M}(\vec{n}_1) = (1, 1,1)$ and $\vec{M}(\vec{n}_2) = (2, 3, 4)$.
The latter allocation has strictly better performance for each group.
However, if $b$ is sufficiently large in the utility in Eq.~\eqref{eq:parity-utility},
then $U(\vec{M}(\vec{n}_1)) > U(\vec{M}(\vec{n}_2))$,
since the penalization term is zero in the first allocation and positive in the second.
In other words, in this example, a preference for parity in performance can lead to worse performance for all groups.

Despite some of the challenges with encoding parity as above, one might still seek to prioritize groups with lower performance to reduce inequitable model performance across groups.
One option for doing so
is to apply a concave transformation to the model performance terms. 
For example, if $U(\vec{M}) = \sum_{k=1}^K \log(M_k)$, the marginal increase in utility is greatest 
for groups with the lowest
model performances, 
encouraging parity.
Yet, nonetheless, a Pareto improvement---in which all groups achieve higher performance---still results in higher utility.  

Non-linear specifications of the utility, such as the two above, which directly penalize inequality may be particularly useful if parity in model performance has large positive externalities to society.  
However, for simplicity, throughout the remainder of this work we focus on the case of linear utility, $U(M) = \sum_{k=1}^K a_kM_k$, which may be suitable in many common applications.

\subsection{An Illustrative example}

\begin{table}[t]
  \begin{tabular}{lcccc|cc}
    \toprule
    Sampling Strategy & $M_1$ & $M_2$ & $M_3$ & $M_4$ & $U_\text{equal}$ & $U_\text{priority}$\\
    \midrule
Equal & 19.5 & 16.7 & 19.5 & 19.5 & 18.8 & 18.9\\
Representative & 19.7 & 16.7 & 19.7 & 17.6 & 18.4 & 18.3\\
Performance Parity & 18.8 & 18.8 & 18.8 & 18.8 & 18.8 & 18.8\\
\midrule
Optimal ($U_\text{equal}$) & 25.5 & 17.3 & 17.3 & 25.5 & 21.4 & -\\
Optimal ($U_\text{priority}$) & 20.0 & 17.3 & 17.3 & 30.0 & - & 22.1\\
  \bottomrule
  \end{tabular}
  \vspace{1mm}
  \caption{Resulting model performances $\vec{M}$ of different strategies for constructing equitable datasets, along with the average model performance, $U_\text{equal}$, across all four countries.  We find that our static methods, equal and representative sampling, result in both lower than possible average model performance in addition to different country-level outcomes. Sampling adaptively from the group with the lowest performance results in equal model performance between countries, but still results in lower than possible total model performance.  We also consider an alternative utility function for a policy-maker who wishes to prioritize country $4$---for example, because it has a more effective regulatory environment around deforestation, or it has been traditionally understudied--- where $U_\text{priority}$ is a weighted average of the country-level model performances with weights $\vec{a} = (1, 1, 1, 1.5)$.  We find that our framework allows us the flexibility to prioritize model performance in country $4$.
  }
  \label{tab:synthetic}
\end{table}

We demonstrate our framework via an illustrative example involving our running computer vision hypothetical.  
Imagine the researcher has a data collection budget of $B = 1000$, and the cost to label an image in each of $K=4$ countries is given by the vector $\vec{c} = (1, 1, 2, 1)$.
Further suppose the group-level learning curves are given by:  
\begin{equation}\label{eq:gamma}
M_k(\vec{n}) = \left(\sum_{j=1}^K \gamma_{k,j} \cdot \vec{n}_j\right)^{\frac{1}{2}},
\hspace{.5cm}
\gamma = 
\begin{bmatrix} 
1 & 0.3 & 0.3 & 0.3 \\
0.3 & 0.5 & 0.3 & 0.3 \\
0.3 & 0.3 & 1 & 0.3 \\
0.3 & 0.3 & 0.3 & 1 \\
\end{bmatrix}.
\end{equation}
\noindent
Because the square root function is  increasing and concave, our specification matches the intuition that more data will increase model performance, albeit at a slowing rate as the size of the training dataset grows. 
Furthermore, the weights $\gamma$ specify that data from any one country helps performance in all the other countries, but at a lesser rate than data from the same country (i.e., for each row of $\gamma$, the diagonal entry is the largest). 
For example, 
\begin{equation*}
M_1 = (n_1 + 0.3 \cdot n_2 + 0.3 \cdot n_3 + 0.3 \cdot n_4)^{\frac{1}{2}},
\end{equation*}
meaning that the model performance for country $1$ scales with the square root of the effective number of training examples, where training examples from other countries are discounted to 30\% that of samples from country $1$.

We now consider a variety of strategies for constructing equitable datasets.  
For example, a model-builder might decide to label an equal number of training samples from each group, resulting in the allocation $\vec{n} = (200, 200, 200, 200)$.  
Alternatively, a model-builder might decide to create a representative dataset, with 
$n_k \sim p_k$,
where the vector
$\vec{p} = (2 \text{ million km}^2, 2 \text{ million km}^2, 2 \text{ million km}^2, 1 \text{ million km}^2)$
denotes the geographical areas of the four hypothetical countries we consider.
Finally, a model-builder might select the allocation so as to ensure parity in performance across the four countries---an outcome that one can achieve by sequentially selecting datapoints from the country with the lowest model performance until the budget is exhausted~\citep{anahideh2020fair,sharaf2020promoting,abernethy2021active}.
The country-level model performances, $M_k$, resulting from these three sampling strategies  are shown in the first three rows of Table \ref{tab:synthetic}.
The second-to-last column in the table shows the average performance across countries $U_\text{equal}(\vec{M}) = \frac{1}{K}\sum_k M_k$, and, for this utility, the penultimate row in the table shows the performance under the utility-maximizing allocation $\vec{n}^* = (500, 0, 0, 500)$.
    
The results in Table~\ref{tab:synthetic} highlight two key points.
First, whereas all 
three common heuristic sampling strategies perform comparably, 
the optimized allocation achieves substantially greater utility.
This gain stems in part from the fact that the static strategies did not account for the differential sampling costs. The optimized strategy, recognizing that the marginal improvement per dollar in country $3$ was lower than in other regions, targeted its budget to the remaining countries. 
Indeed, in the optimal allocation, no samples were collected from two of the four countries.
By avoiding sampling from relatively expensive countries, 
the optimal strategy was able to acquire more total datapoints---for example, while the equal sampling strategy acquired 800 datapoints, the optimal strategy acquired 1,000.
Second, even though the optimal allocation did not collect any samples from countries 2 or 3, it still 
was able to achieve reasonable performance in those regions, given the inter-country learning effects.
In fact, in country 2, the optimal strategy achieved higher performance than both the equal sampling and representative sampling approaches.
Thus, although all three of the heuristic sampling approaches seem a priori reasonable, they result in quite different overall and country-level performances, 
demonstrating the value of formalizing one's goals for a dataset,
and then optimizing for those objectives. 

Finally, we consider an alternative hypothetical scenario where the model-builder wants to intervene to make the model more inclusive for country $4$, perhaps due to a stronger regulatory environment making deforestation interventions more effective there, or because past research has not included country $4$.
To encode these preferences, the model-builder sets $U_{\text{priority}}(\vec{M}) = \frac{1}{\sum_{k=1}^K a_k}\sum_{k=1}^K a_kM_k$, where $\vec{a} = (1, 1, 1, 1.5)$.  
The optimal strategy under this setting is to choose $\vec{n}^* = (143,0,0,857)$, moving some of the samples in our original optimal solution from group $1$ to group $4$ to increase the model performance for group $4$. (See the last row of Table~\ref{tab:synthetic} for country-level performance.)  
Whereas traditional approaches to dataset construction do not actively consider such preferences, our framework allows for the flexibility to pose and optimize for these trade-offs.

\section{Finding optimal allocations} \label{sec:opt}

When the learning curves $M_k$ are known---and the learning curves are concave and utility is linear---standard techniques from convex optimization allow one to efficiently compute optimal allocations. 
However, in practice, the learning curves are not usually known a priori, before data are collected.  
In this scenario, it is useful to draw a distinction between situations where
sampling is done in one shot, with the allocation determined prior to any data collection,
and where
sampling can be done \emph{sequentially}, in which the model-builder can collect samples one at a time and use information gleaned from the current sample
to decide which group to sample from next.
Many real-world scenarios may in fact lie somewhere between these two extremes, where batches of data are collected at a time and the model builder can update their sampling strategy between batches.  
In the sequential or batch-sequential settings, one can  estimate the learning curves at each step using the existing training data, in addition to potentially using prior knowledge from training similar models.  
Based on this information, one can then judiciously select the next group to sample from.

Here we present a greedy allocation algorithm,
which only requires local estimates of the marginal increase in model performance, rather than estimates of the full learning curve.
In practice, these local estimates can be obtained by observing how model performance previously changed as more data were added, an approach we demonstrate in Section \ref{sec:prs} in the context of constructing polygenic risk scores. 
We start by defining a step size $s$, which can be viewed as the number of dollars we spend at each step of the algorithm.  
Then, given a current allocation $\vec{n}$, the next datapoint is selected from the group that is expected to increase utility the most.
That is, the next group $i^*$ is selected to satisfy:
\begin{equation}\label{eq:greedy}
i^* \in \argmax_{1 \leq i \leq K} \hat{U}\left(\vec{n}+\frac{s}{c_i}1_i\right) = \argmax_{1 \leq i \leq K} \sum_{k=1}^K a_k \hat{M}_k\left(\vec{n}+\frac{s}{c_i}1_i\right),
\end{equation}
where $\hat{U}$ and $\hat{M}$ reflect the model-builder's current estimates.  
Importantly, to select $i^*$ one only needs accurate local knowledge of the learning curves (i.e., the likely performance gain for an additional sample from that group).
Algorithm~\ref{alg:cap} outlines the process of applying this approach.   

\begin{algorithm}[t]
\caption{Greedy algorithm to construct an equitable dataset.}\label{alg:cap}
\begin{algorithmic}
\State $\texttt{ALLOC} \gets \texttt{START}$\Comment{\texttt{ALLOC} is an array with element $k$ equal to the current number of samples from group $k$.}

\While{$\texttt{ALLOC} \cdot \texttt{COST} < B$} \Comment{Enforce the budgetary constraint, where $\cdot$ is the dot product}

\For{$k \gets 1$ to $K$}    
\State $\texttt{PRIORITY}[k] \gets \texttt{ESTIMATE\_MARGINAL}(\texttt{ALLOC},k)$
\EndFor
\State $\texttt{GROUP} \gets \argmax_i \texttt{PRIORITY}[i]$
\State $\texttt{ALLOC}[\texttt{GROUP}] = \texttt{ALLOC}[\texttt{GROUP}] + \frac{\texttt{STEP\_SIZE}}{\texttt{COST}[i]}$
\EndWhile

\end{algorithmic}
\end{algorithm}

In Theorem \ref{thm:optimality}, we show that if the true forms of the learning curves are concave and the data from one group do not affect the derivative of model performance in the other groups, then the greedy strategy finds the optimal solution given only knowledge of local marginal improvements.  
We note that this condition holds in the special case when the model-builder trains separate models for each group,
as is often done in our motivating example of polygenic risk estimation.

\begin{theorem}
\label{thm:optimality}
Suppose the learning curves are concave increasing and utility is linear.
Further suppose that the partial derivatives of the learning curves have no cross-group effects, meaning that
if $\vec{p}_j = \vec{q}_j$ then
\begin{equation*}
\frac{\partial M_k(\vec{p})}{\partial n_j} = \frac{\partial M_k(\vec{q})}{\partial n_j} \quad \text{for} \ 1 \leq k \leq K.
\end{equation*}
Then the greedy algorithm, initialized at the zero allocation $\vec{n} = 0$ with a given step size $s$, maximizes $U$ over all feasible allocations where $n_k$ is a multiple of $\frac{s}{c_k}$ for all $k$.
\end{theorem}
\begin{proof}
First, we define the marginal improvement of utility of the $j$th batch from the $i$th group:

$$m_{i,j} = \sum_{k=1}^K a_k[M_k(n_1, ..., n_i=j\frac{s}{c_i}, ..., n_K) - M_k(n_1, ..., n_i = (j-1)\frac{s}{c_i}, ..., n_K)] = \sum_{k=1}^K a_k o_{ijk}.$$

By the condition on the partial derivatives, the difference $M_v(n_1, ..., n_i=j\frac{s}{c_i}, ..., n_k) - M_v(n_1, ..., n_i = (j-1)\frac{s}{c_i}, ..., n_k)$ depends only on the value of $n_i$, and is independent from all other elements of $\vec{n}$.  
Then, we note that any allocation $A$ can be written as $\{(i, j)\}$, where $(i, j) \in A$ implies that the allocation includes the $j$th batch from group $i$. 
The model-builder's utility for an allocation can be written $$U(A) = \sum_{i, j \in A} m_{i,j}.$$

Let $d = \frac{B}{s}$ be the number of batches that the model-builder will purchase.  
Then, an upper bound on the possible utility of the allocation is the sum of the $d$ highest $m_{i,j}$.  
We will show that the greedy algorithm at each step chooses a batch $(i,j)$ corresponding to the highest value of $m_{i,j}$ out of all batches $(i,j)$ not included in the greedy allocation, implying that it achieves that upper bound.  

Say that our greedy algorithm at step $t$ chooses to sample batch $(i_t, j_t)$ and batch $(i_\ast, j_\ast)$ has not been sampled.  .

Case 1: $i_t = i_\ast$.  Then, $j_\ast > j_t $, since the greedy algorithm has already sampled $(i_t, 1)...(i_t,j_t)$.  

\begin{equation*}
 m_{i_t,j_t} = \sum_{v=1}^k a_v o_{i_tj_tv}> \sum_{v=1}^k a_v o_{i_tj_*v} = m_{i*,j_*} 
\end{equation*}

where the inequality is given by the concavity of the learning curves and that $j_\ast> j_t$.

Case 2: $i_t \neq i_\ast$.  Let $j'$ be the number of batches the greedy algorithm has sampled from group $i_\ast$.  
Then, $m_{i_t,j_t} >= m_{i_\ast, j'+1} >= m_{i_\ast, j_\ast}$, where the first inequality comes from the fact that our algorithm is greedy and the second comes from the concavity of the learning curves.
\label{optimality}
\end{proof}

Theorem~\ref{thm:optimality} shows that the greedy algorithm is provably optimal when the learning curves do not have cross-group effects.
However, numerical experiments suggest that the greedy algorithm is optimal in a wide-variety of settings beyond those satisfying the conditions of the theorem.
Consider, for instance, our running
computer vision example.
The learning curves defined by Eq.~\eqref{eq:gamma}
violate the assumptions of Theorem~\ref{thm:optimality}, 
as the marginal learning rates in each group depend on the number of samples currently collected in all other groups.
Nonetheless, 
we find that the greedy algorithm 
achieves the optimal utility for both the equal utility and the prioritized utility functions, as shown in Table~\ref{tab:synthetic}.

To further investigate the behavior of the greedy algorithm, we conducted an extensive set of numerical experiments.
We specifically considered random problem instances in which the number of groups $K$ varied from 2 to 10, 
costs $\vec{c}$ were randomized such that $c_k \sim \text{UNIF}(0, 1)$,
weights of the utility function $\vec{a}$ were randomized such that $a_k \sim \text{UNIF}(0, 1)$,
and the learning curves were randomized so that: 
\begin{equation}\label{eq:gamma2}
M_k(\vec{n}) = f\left(\sum_{j=1}^K \gamma_{k,j} \cdot \vec{n}_j\right),
\hspace{.5cm}
\gamma_{k,j} \sim \text{UNIF}(0, 1),
\end{equation}
for two functional forms, $f(x) = \log(x)$ and $f(x) = \sqrt{x}$.  
Under all circumstances, we find that the mean absolute difference between the solutions found via convex optimization and the greedy algorithm approaches $0$ as the step size goes to $0$.  
These numerical findings suggest that the greedy algorithm is a robust approach to finding optimal allocations under a wide range of conditions;
analytically characterizing the algorithm's properties would be an interesting direction for future work.  

\section{An application to polygenic risk scores}\label{sec:prs}
\subsection{Background}
We now transition from our simple, stylized deforestation example to a more detailed application involving polygenic risk scores (PRSs).
Polygenic risk scores are statistical models which use the presence of genomic variants in one's DNA sequence in order to estimate risk for developing a complex disease.  
PRSs have been found to be predictive for many complex genetic diseases such as coronary artery disease and Type 2 diabetes \cite{khera2018genome}, and are believed to be promising tools for risk stratification for health interventions more broadly.

However, PRSs have been found to exhibit disparities in performance across groups defined by ancestry \cite{de2018polygenic}.  
Specifically, many PRSs have been found to perform worse in people of African descent. 
The main cause of this performance gap is thought to be a lack of ancestral representation in genome-wide association studies (GWAS), from which the datasets to train polygenic risk scores typically come.  
To date, about 52\% of all GWAS were conducted in populations of European descent compared to 10\% in populations of African descent, and 78\% of individuals who appear in GWAS are of European ancestry compared to 2\% of African ancestry~\cite{sirugo2019missing}. 
Furthermore, 72\% of individuals in GWAS were recruited from only three countries: the United States, the United Kingdom, and Iceland \cite{mills2019scientometric}. Additional work has shown this lack of diversity in GWAS could result in over- or under-estimation of genetic disease risk in understudied populations and could potentially exacerbate health disparities \cite{sirugo2019missing, egede2006race}.

Despite calls for additional representation for non-European ancestries in GWAS and PRSs \cite{sirugo2019missing, popejoy2016genomics, martin2017human}, it is still unclear exactly how a model-builder interested in constructing a PRS should allocate their limited funding between gathering genomic data from people of different ancestries. 
To demonstrate how our framework might be applied in this setting,
we use a simulation framework developed by domain experts \cite{cavazos,kelleher2016efficient} to first generate a synthetic population 
of people with different ancestries,
and then train PRS models under various sampling strategies. 

\subsection{Simulation details}
Following \citet{cavazos}, we simulated genomes of 200,000 people of European (CEU) and African (YRI) descent, along with the presence of a phenotype (disease) with 5\% prevalence in both populations.
We used the simulated data to train separate PRSs in each population, evaluating model performance over a variety of training allocations (see Appendix~\ref{app:mm} for further details).
Out of the 10,000 people who will get the disease (called ``cases'') 
and 190,000 people who will never get the disease (called ``controls''),
we chose a random sample of 5,000 cases and 5,000 controls to be the obtainable training data.
Trained models were evaluated on a holdout test set comprised of the remaining 5,000 cases and 95,000 other randomly selected controls.

For our hypothetical disease, we imagine there is a health intervention that has cost $c$ and benefit $b$.  
That is, for an individual $i$, the intervention has value
\begin{equation}
V = bd_i - c,
\end{equation}
where $d_i$ is an indicator variable for whether the person will eventually get the disease. 
If $\hat{p}$ denotes one's estimated likelihood of developing the disease, based on the available genomic data, the expected value 
of intervening is:
\begin{equation}
\mathbb{E}[V \mid \hat{p}] = b \hat{p} - c.
\end{equation}

Suppose the cost and benefit of the intervention are given by 
$c = 5$ and $b = 100$,
constants which we use for the remainder of our analysis.\footnote{We note that the cost here could either be monetary or health-related, such as radiation exposure from X-rays.}
Then the expected value of the intervention is positive for individuals for whom $\hat{p} > 0.05$, negative 
for $\hat{p} < 0.05$,
and zero for $\hat{p} = 0.05$.
Given the base prevalence of the disease is $5\%$, 
the expected utility of intervening on a random person is 0.
However, if the model-builder is able to identify and selectively treat individuals at high risk for the disease, the intervention can yield positive value.

Given a predictive model $\hat{f}_{T_{\vec{n}}}$ trained on
the genomic dataset $T_{\vec{n}}$,
the value-maximizing intervention strategy is to treat those with estimated risk greater than 0.05.
We define the group-level model performance of a training allocation to be the expected value from applying this decision rule on a random member of the group:
\begin{equation}
M_k(\vec{n}) = \mathbb{E}_{T_{\vec{n}}}\mathbb{E}_{X^0, Y^0}[G(Y^0, \hat{f}_{T_{\vec{n}}}(X^0)) \mid X^0_g=k],
\end{equation}
where
\begin{equation}
G(X_0, Y_0) = \mathbbm{1}_{\hat{f}_{T_{\vec{n}}}(X^0) > 0.05}\cdot(Y^0b-c),
\end{equation}
and the pair $(X^0, Y^0)$ represents the genomic markers and eventual disease status of a random individual belonging to group $k$.

\subsection{Constructing equitable datasets}
We consider a scenario in which the model-builder has budget $B = 5000$, and where samples from each group cost 1 unit, 
where a sample is a single case-control pair.  
The model-builder begins initially with 500 samples from each group, and must then choose how to allocate their budget in increments of $s=100$.   
We evaluate a variety of policies for allocating the budget between sampling from CEU (European descent) and YRI (African descent) data.  
We specifically consider two static policies: (1) representative sampling, where the proportion of training data from both groups mirrors their proportion in the overall population;\footnote{We assume the intervention is being done in the United States, and use the proportion of Black and non-Hispanic white individuals in the 2020 census~\cite{census2020}.} 
and (2) equal sampling, where $n_\text{YRI} = n_\text{CEU}$.  
To adhere to the step-size $s$, we restrict these static policies to the closest allocations with $n_\text{CEU}$ and $n_\text{YRI}$ being multiples of $s$. 
We also consider two active sampling strategies, which allocate the budget sequentially:
(1) sampling from the group which currently has lower model performance, in an effort to achieve performance parity \cite{abernethy2021active};
and (2) our greedy adaptive sampling algorithm discussed in Section~\ref{sec:opt}.  

\begin{algorithm}[t]
\caption{Implementation of \texttt{ESTIMATE\_MARGINAL} for PRS application}\label{alg:est}
\begin{algorithmic}
\State $\texttt{ESTIMATE\_MARGINAL(ALLOC, k)}:$
\Indent
\State $X \gets \texttt{SEQ}(\texttt{MAX}(\texttt{START}[k], \texttt{ALLOC}[k] - (m-1) \cdot \texttt{STEP\_SIZE}), \texttt{ALLOC}[k], \texttt{STEP\_SIZE})$
\State $Y \gets [\texttt{MODEL\_PERF}(x) \texttt{ for x in X}]$
\State $\hat{\beta}, \hat{\texttt{SE}}_\beta \gets \texttt{LINEAR\_REGRESSION(Y, X)}$
\State $Z \gets N_+(\hat{\beta}, \hat{\texttt{SE}_\beta^2})$
\State $\texttt{RETURN}(Z \cdot
\frac{\texttt{STEP\_SIZE}}{\texttt{COST}[k]})$
\EndIndent
\end{algorithmic}
\end{algorithm}

To apply our adaptive sampling method (Algorithm \ref{alg:cap}), the model builder needs to
estimate the marginal improvement in utility $U\left(\vec{n} + \frac{s}{c_k}1_k\right) - U(\vec{n})$ for each group $k$ given their current allocation $\vec{n}$.
We outline our implementation of this estimation problem in Algorithm \ref{alg:est}. 
Our method for estimating the marginal improvement is to keep track of our model performance at each allocation, and then construct a local approximation of the learning curve via linear regression, using the last $m = 5$ measurements of model performance (or all the available points, if fewer than five models have been trained for a given group).  
The choice of $m$ can be thought of as a bias-variance trade-off, where higher $m$ leads to bias because the true slope is decreasing but low $m$ leads to variance because the individual observations of model performance are noisy.  Then, for each group $k \in \{\text{YRI}, \text{CEU}\}$ we get both an estimate $\hat{\beta}_k$ of the increase in performance per training sample,
and a standard error $\hat{SE}_{\beta_k}$ of that estimate.
To account for noise in our estimate, we select the next group to sample based on a draw $\tilde{\beta}_k \sim N_+\left(\hat{\beta}_k, \hat{SE}_{\beta_k}^2\right)$, 
where $N_+$ is the truncated normal distribution, bounded from 0 to $\infty$.
This procedure can be thought of as analogous to Thompson sampling, with a prior that more data cannot decrease model performance.
We apply this stochastic method due to challenges in estimating model performance.  
In a setting where model training was computationally inexpensive, one might alternatively address this problem by bootstrapping the collected data and training and evaluating several models at each training size; in our setting, though, that approach was not feasible, as PRSs are computationally intensive to train.

\begin{figure}[t]
    \centering
    \includegraphics[width=.5\columnwidth]{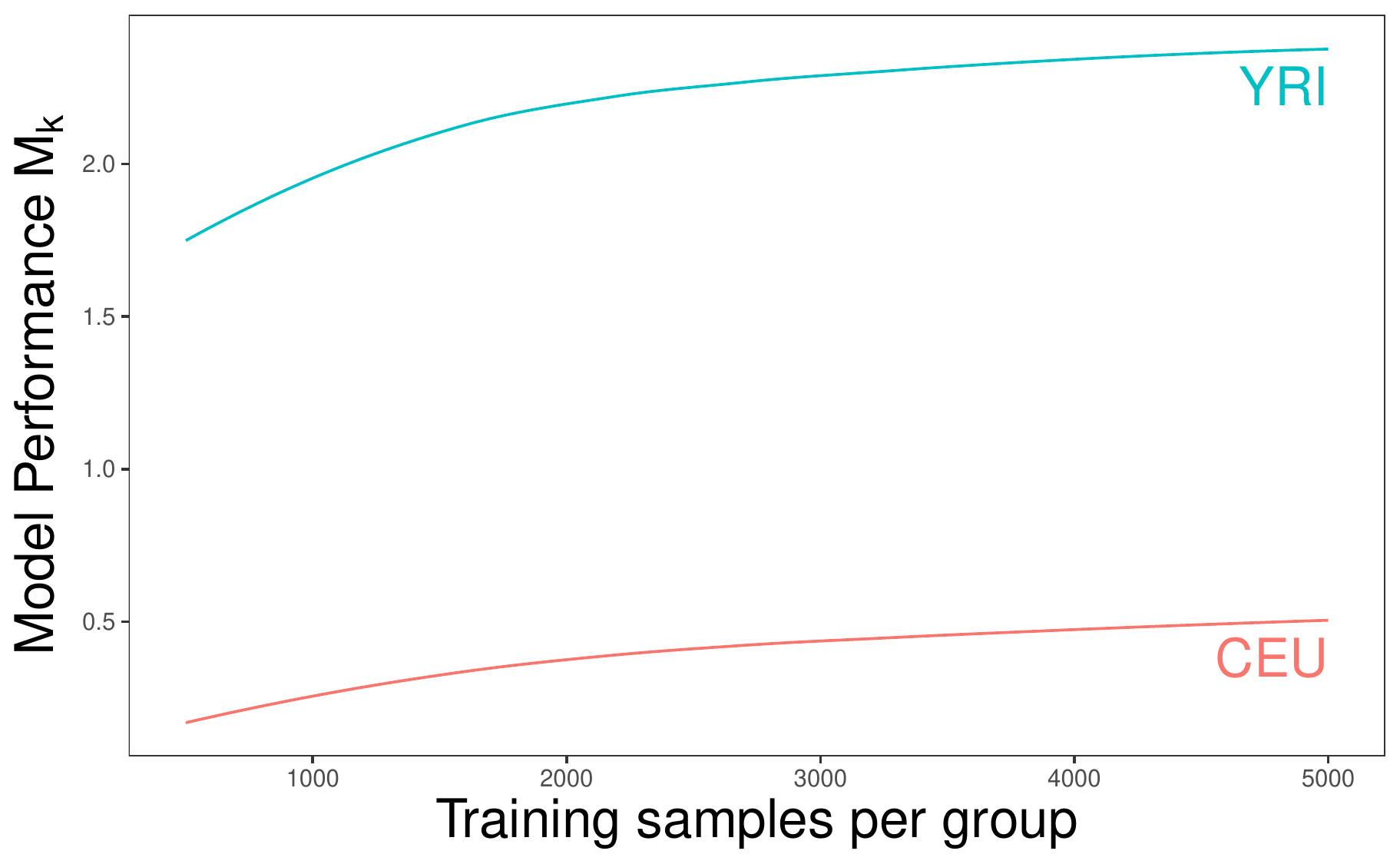}
    \caption{The learning curves of our polygenic risk score model, where each point represents the average per-capita utility for members of a group if the training set contains $x$ number of people in that group.  The PRS both starts with a better performance and improves faster for the YRI group.  
    }
\label{learning_curve}
\end{figure}

\begin{figure}[t]
    \centering
    \includegraphics[width=.5\columnwidth]{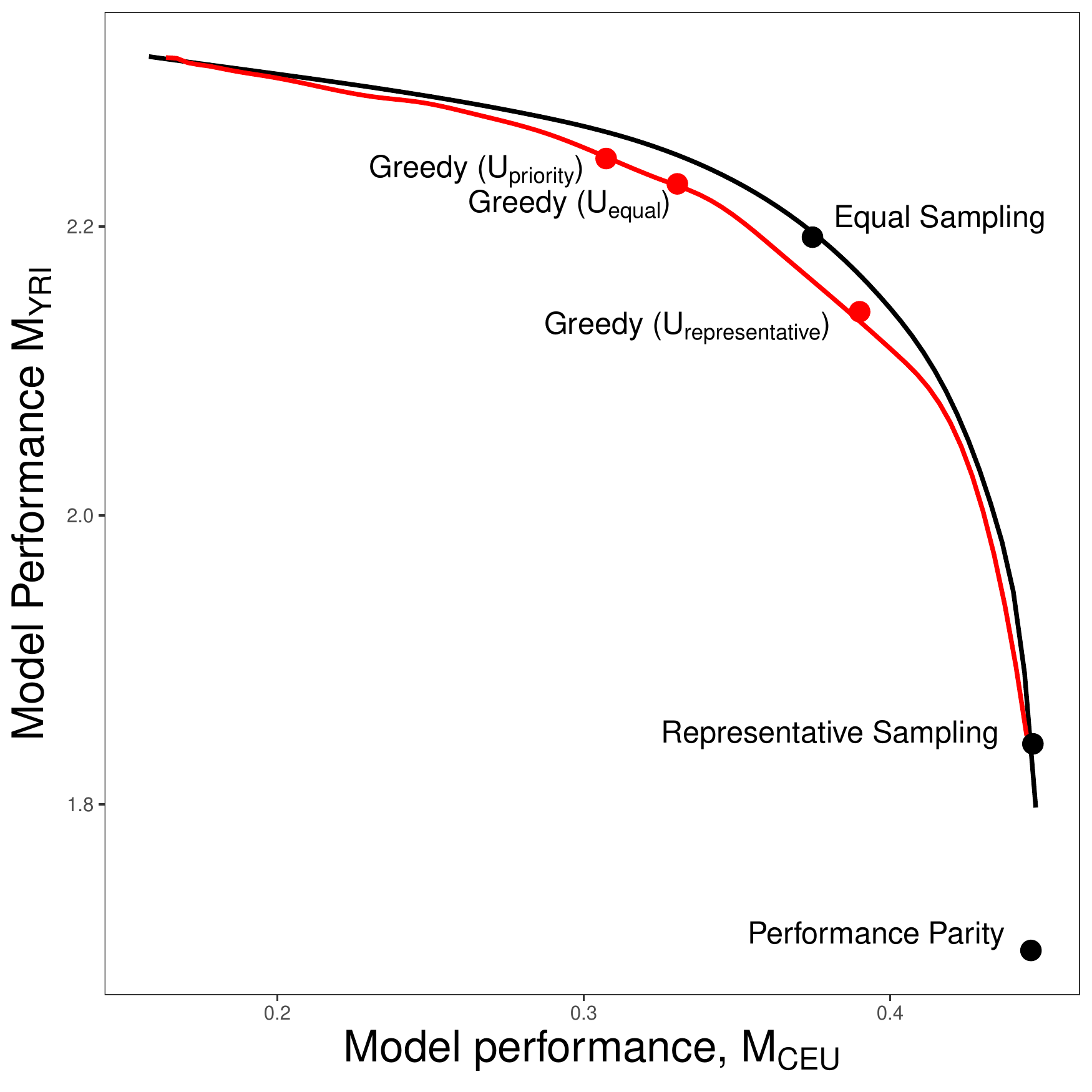}
    \caption{The performance Pareto frontier of our setting, where each point corresponds to an allocation of our budget $B = 5,000$ such that $N_\text{CEU}$ + $N_\text{YRI} = B$, $N_\text{CEU} \geq 500, N_\text{YRI} \geq 500$.  Going left to right, as we increase the proportion of our allocation towards gathering samples from the CEU group, the model performance $M_\text{YRI}$ decreases and $M_\text{CEU}$ increases.  We evaluate strategies for constructing equitable datasets, finding that the greedy adaptive sampling algorithm is able to find near-optimal policies under a wide range of utility specifications.  
    }
\label{perf_plot}
\end{figure}

\subsection{Results}
Following the above setup,
Figure \ref{learning_curve} shows the learning curves of the two group-level performances $M_\text{YRI}$ and $M_\text{CEU}$ as a function of the size of the dataset used to train each group's models, $N_\text{YRI}$ and $N_\text{CEU}$, across 50 simulations. 
In our hypothetical scenario, we find, for a fixed number of training samples, that the polygenic risk score for individuals of African ancestry both starts off with a higher performance at the minimum 500 samples ($M_\text{YRI} = 1.69, M_\text{CEU} = 0.158$), and has improved roughly by twice as much at the maximum possible 5,000 samples
($M_\text{YRI} = 2.31,
M_\text{CEU} = 0.447,
\Delta_\text{YRI} = 0.62,  \Delta_\text{CEU} = 0.289$).
We note that this phenomenon is a consequence of the way we selected the parameters of our simulation; it is intended only as an illustrative example, and is not representative of broader trends (see Appendix \ref{app:mm} for details).

In our setting, the model builder can choose any training allocation $\vec{n} = (N_\text{CEU}, N_\text{YRI})$ such that $N_\text{CEU} + N_\text{YRI} \leq 5000$ (with $N_\text{CEU} \geq 500$ and $N_\text{YRI} \geq 500$, given the initial dataset).
The black line in Figure \ref{perf_plot} shows the tradeoff between $M_\text{CEU}$ and $M_\text{YRI}$ over all possible choices of $\vec{n}$ such that $N_\text{CEU} + N_\text{YRI} = 5000$.
Along this frontier, we plot both the model performances resulting from an equal sampling policy ($\vec{n} = (2500,2500)$), along with that of a representative sampling policy ($\vec{n} = (3300,700)$), which has more samples from $M_\text{CEU}$ because we set our hypothetical intervention in the United States.
We find that the resulting policies differ drastically in where along the frontier they fall.
In particular, relative to the equal-sampling strategy, the representative-sampling strategy (which samples more individuals of European descent)
implicitly sacrifices substantial gains in model performance for individuals of African descent for more modest performance gains for individuals of European descent.

In addition to these two commonly employed static strategies, we plot the results of a third, dynamic strategy, which attempts to equalize model performance by sampling from the worst-performing group at each step.
In our simulation, this strategy always samples from the CEU group, since it has lower performance at all allocations $\vec{n}$.  
This strategy---labeled ``performance parity'' in Figure~\ref{perf_plot}---results in trading off a large amount of $M_\text{YRI}$ for a very small amount of $M_\text{CEU}$. 
Indeed, because the marginal increase of $M_\text{CEU}$ per training sample becomes essentially zero, the point does not even appear on the frontier in Figure~\ref{perf_plot}. 

Where these three sampling strategies---equal, representative, and performance parity---lie on the frontier will, in general, vary depending on the structure of the learning curves.
For instance, in a scenario where CEU performance increased faster than YRI performance, representative sampling would result in trading off a relatively lower amount of YRI performance for a higher amount of CEU performance relative to our setting.  
In particular, as the two static sampling strategies consider only the composition of the training dataset and not its impact on model performance, they will be unstable in their valuations for group-level model performances in different circumstances.

Depending on the training allocation, one can land anywhere on the black frontier in Figure~\ref{perf_plot}. 
But where exactly one might choose to land depends on application-specific considerations. 
We now imagine a model-builder with utility that is linear in the group-level model performances:
$$U(M_\text{CEU}, M_\text{YRI}) = a_\text{CEU}M_\text{CEU} + a_\text{YRI}M_\text{YRI},$$ 
where the weights $\vec{a} = (a_\text{CEU}, a_\text{YRI})$ are non-negative and determine where on the frontier the optimal policy falls.
For various settings of the weights, 
we run our greedy sampling algorithm, initialized at the allocation $\vec{n} = (500, 500)$. 
The red line in Figure \ref{perf_plot}
shows the resulting model performances while we vary $\frac{1}{1000} \leq \frac{a_\text{CEU}}{a_\text{YRI}} \leq 1000$.  We find that our approach is able to identify near-optimal policies over a wide range of specifications for utility, with a small amount of loss due to noise in our estimation of the marginal improvement (Algorithm \ref{alg:est}). 

Finally, we consider the question of how a model-builder might decide to set $\vec{a}$ in their specification of utility.  
If we take the point of view that the benefit $b$ and cost $c$ of the intervention are in units of lives saved, setting $\vec{a} = (p_\text{CEU}, p_\text{YRI})$ to be proportional to the size of the group in the overall population optimizes the total number of lives saved.  The resulting policy given by these weights is labeled ``Greedy $(U_{\text{representative}})$'' in Figure \ref{perf_plot}.
We note that this sampling strategy is distinct from representative sampling, which sets the composition of the training dataset to be proportional to the size of the group, instead of the valuations on model performance.  
In particular, when optimizing for $U_{\text{representative}}$, the greedy strategy recognizes that although the YRI group is smaller, it has higher marginal gains in performance, and thus samples more heavily from that group than does the representative-sampling strategy.  
In this case, the greedy strategy optimizing for $U_\text{representative}$ has both higher group-level performance for the YRI group as well as higher overall performance than  representative sampling.   

Another natural choice might be to set $\vec{a} = (1, 1)$, so that model performance is valued equally among the two groups.  This strategy, labeled ``Greedy $(U_{\text{equal}})$'' in Figure~\ref{perf_plot}, results in drawing more samples from the YRI group compared to the greedy strategy with representative weights, since the size of the groups is ignored and the YRI group, which has a higher marginal improvement per training sample, is prioritized.   We note that this notion of equal value for group-level model performance is different than both model parity (the closest strategy to parity results in sampling only from CEU), and equal sampling, which enforces equality in the training set composition instead of the model performance valuations.

Finally, a model-builder might take the point of view that PRSs have traditionally excluded those of African descent~\citep{de2018polygenic}, and put model performance for that group at a premium by setting $\vec{a} = (1, 1.50)$.  The resulting model performances from running the greedy algorithm with these preferences is labeled ``Greedy $(U_{\text{priority}})$'', and is very close to the resulting performances for the last specification $\vec{a} = (1, 1)$, reflecting that moving further toward the upper left of the plot requires a large trade-off in $M_\text{CEU}$ to achieve a small gain of $M_\text{YRI}$. This pattern is a function of this particular learning curve, and, in a different setting, the priority might result in a much different allocation than the greedy sampling strategy with equal weights.  
\section{Discussion}
Statistical models across a wide variety of domains have been shown to exhibit disparities in model performance, in part due to lack of representation in the datasets they are trained on. To mitigate this problem, we present a framework for a model-builder to specify a preference over resulting group-level model performances, and then formalize the task of constructing a dataset as a constrained optimization problem. We present an adaptive sampling algorithm for constructing datasets which takes into account both the structure of how data from one group affects model performance in the others in addition to the cost of acquiring data.  We showed both empirically and analytically that taking these two factors into account allows our adaptive algorithm to identify near-optimal solutions, and can avoid some of the unintended consequences that can arise with static sampling methods such as equal or representative sampling.
Finally, we demonstrated how our framework allows for the model-builder to efficiently intervene when circumstances dictate that model performance should be prioritized for a given group: for example, due to  traditional models underserving a group, or model performance better translating to impact in that group.  

Our findings can inform practitioners as well as policymakers seeking clarity on what would constitute sufficiently representative and inclusive datasets. In particular, our findings demonstrating drawbacks of static sampling methods suggest that future guidelines or requirements around dataset representativity \cite{act2021proposal} should take care not to codify sampling approaches that are insufficiently flexible in considering all the factors surrounding the construction of a dataset, including the effects of sampling strategies on actual model performance.  

We conclude by noting some important limitations of our analysis.  
First, although our greedy algorithm appears to work well for one natural family of learning curves $M_k$, it may not be an effective approach in every instance. 
There are many types of data and many methods of training models using such data, which can result in a variety of different structures for the learning curves.  
For instance, a deep learning approach to training PRSs might use all available data for a single model instead of training separate models such as in our example~\citep{badre2021deep}. 
A promising direction for future work is to consider how our framework might be applied to a variety of different approaches to building models in different domains. 
Second, our greedy algorithm requires the model-builder to estimate the marginal improvements in $M_k$ at each step, which can be statistically and computationally challenging, especially when there are limited data for certain groups or when training models requires significant computing resources.
Third, in this work we considered a particular specification of utility, but others may be appropriate depending on the setting.  For instance, if data are collected with the purpose of being used in the future in addition to training a model, the utility function might also encode the value of the data for training future models.  
Finally, this method may not be applicable in circumstances where a training dataset cannot be responsibly expanded, such as data regarding individuals' interaction with police or the criminal legal system, or where privacy interests are determined to outweigh model performance or fairness goals; in such cases, approaches leveraging synthetic data or experimenting with alternative modelling options may be more appropriate to address fairness concerns. 

We see our work as part of a broadening of how machine learning practitioners operationalize algorithmic fairness.
In addition to approaches tailored to improving the equity of models trained on static datasets, it is important 
to consider issues that arise at various stages of the training and deployment of statistical models, including constructing equitable training sets ~\citep{matise2011next, aviddataset,galvez2021people,hazirbas2021towards,matise2011next, aviddataset,galvez2021people,hazirbas2021towards}, interventions to bolster model performance for traditionally underserved groups, such as screening~\citep{cai2020fair,noriega2019active, bakker2019fairness}, 
and designing more equitable interventions given a set of risk scores~\citep{L2BF}.  
We hope our work will help support these ongoing efforts.

\begin{acks}
We thank Taylor Cavazos and John Witte for helpful conversations regarding polygenic risk scores, and for developing the original PRS simulation framework that we used in our analysis.  We also thank Jovani Gutierrez for assistance with running experiments. 
\end{acks}

\bibliographystyle{ACM-Reference-Format}
\bibliography{ms}

\appendix

\section{Genetics Glossary}

\textbf{Causal Variant}\indent In the context of Genome-Wide Association Studies (GWAS), causal variants are genetic variants that have a biological effect on polygenic diseases (e.g. coronary heart disease, cancer, diabetes), which are diseases caused by the combined effects of multiple genes. 

\noindent \textbf{Genome} \indent An individual organism's complete set of genetic instructions; DNA. 

\noindent \textbf{Genome-Wide Association Studies (GWAS)} \indent Studies performed for use in genetics research to identify genetic variants present at a higher frequency in individuals with a specific trait (e.g., a disease) in a population. 

\noindent \textbf{Genotype} \indent A subset of genes in an individual organism, which can contribute to a phenotype. 

\noindent \textbf{Minor Allele Frequency (MAF)} \indent The proportion of time the allele that appears less frequently in a given population occurs.

\noindent \textbf{Phenotype} \indent Observable traits such as height, eye color, and presence of a disease in an individual. 

\section{PRS detailed materials and methods}\label{app:mm}

Following \citet{cavazos}, we simulate European (CEU) and African (YRI) ancestry genotypes for chromosome 20, simulating genomes of 200,000 people of European descent and 200,000 people of African descent.  We then computed the minor allele frequency (MAF) for each population throughout the simulated genotypes and ranked the genotypes by the difference $MAF_\text{YRI} - MAF_\text{CEU}$. We chose the top ranked 1,000 variants as our casual variants to simulate a disease where a PRS might have more predictive power in one group, in this case those of African descent.  

For each selected causal variant $i$, we continue following Cavazos et al. \cite{cavazos}, drawing an effect size $\beta_i \sim N(0, \frac{h^2}{1000})$, where $h = \frac{1}{2}$ controls the trait heritability.  We then compute the total genetic liability for individual $j$ as $X_j \sim \sum_{i=1}^{1000} \beta_i g_i$, where $g_i$ is an indicator variable for if the genetic variant appears at location $i$ in person $j$'s DNA sequence. Then, we compute the non-genetic effect as $\epsilon_j \sim N(0, 1-h^2)$.  After both $X$ and $\epsilon$ are standardized ($G = \frac{X-\mu_X}{\sigma_X} * \sqrt{h^2}, E = \frac{\epsilon - \mu_\epsilon}{\sigma_\epsilon} * \sqrt{1 - h^2})$ they are added to obtain the total trait liability (G+E). Each individual is then ranked by their total trait liability and the top 5\% of individuals in the CEU and YRI populations are given the phenotype (disease), $Y = 1$, with the rest having $Y = 0$.  

To train the polygenic risk scores in the CEU and YRI populations, a GWAS is conducted to select genetic variants for inclusion. Genetic variants were selected via a standard two-step process of p-value thresholding and clumping. For each genotype with a MAF $> 1\%$, we compute an odds ratio and assess statistical significance with a chi-squared test, with all genotypes with $p < .01$ being selected. We further filtered the genotypes via clumping to remove highly correlated adjacent genotypes, removing genotypes within a $1$ MB window that have a Pearson correlation of $r = .2$.  For each individual, their  empirical PRS was given by $\sum_{i=1}^V \log(OR_i) g_i$, where $V$ is the number of remaining variants after the clumping + thresholding process, $OR_i$ is the odds ratio for the $i$th selected variant, and $g_i$ is an indicator variable for whether the variant is present in that person.  Lastly, we use Platt scaling to convert each PRS for an individual to a probability of disease risk. 

\end{document}